\documentclass[letterpaper, 10 pt, conference]{ieeeconf}

\IEEEoverridecommandlockouts
\usepackage{settings/preamble}

\usepackage{xcolor}

\newcommand{\lina}[1]{{{\textcolor{blue}{Lina: #1}}}}

\renewcommand{\b}[1]{\left(#1\right)}

\newcommand{\boldtitle}[1]{\medskip\noindent\textbf{#1}\;}

\renewcommand{\R}{\mathbb{R}}
\newcommand{\hK}{\widehat{K}}
\newcommand{\hKBC}{\widehat{K}_{\mathrm{BC}}}

\newcommand{\hKMS}{\widehat{K}_{\mathrm{PIL}}}
\newcommand{\omegaBC}{\omega_{\mathrm{BC}}}
\newcommand{\omegaPIL}{\omega_{\mathrm{PIL}}}
\newcommand{\hG}{\widehat{G}}
\newcommand{\Ks}{K^{\star}}

\usepackage{tikz}
\usetikzlibrary{calc,patterns,
                 decorations.pathmorphing,
                 decorations.markings, 3d}

\newbool{arxiv}
\setbool{arxiv}{true}

\newcommand\setversion[1]{%
    \def\tempa{#1}%
    \def\tempb{student}%
    \ifx\tempa\tempb
        \setbool{arxiv}{true}%
    \else
        \def\tempb{teacher}%
        \ifx\tempa\tempb
            \setbool{arxiv}{false}%
        \else
            \errmessage{Unknown value for arxiv: #1}%
        \fi
    \fi
}

\begin{document}


\title{\LARGE \bf A Model-Based Approach to Imitation Learning through Multi-Step Predictions}

\author{Haldun Balim, Yang Hu, Yuyang Zhang, Na Li 
\thanks{Balim, Hu, Zhang, and Li are affiliated with the School of Engineering and Applied Sciences at Harvard University, Massachusetts, USA (e-mail: hbalim@fas.harvard.edu, yanghu@g.harvard.edu,yuyangzhang@g.harvard.edu, nali@seas.harvard.edu).}
\thanks{This work is funded by NSF AI institute: 2112085, NSF ECCS-2401390, and ONR N000142512173.}
}

\maketitle
\thispagestyle{empty}
\pagestyle{empty}


  \begin{abstract}
Imitation learning is a widely used approach for training agents to replicate expert behavior in complex decision-making tasks. However, existing methods often struggle with compounding errors and limited generalization, due to the inherent challenge of error correction and the distribution shift between training and deployment. In this paper, we present a novel model-based imitation learning framework inspired by model predictive control, which addresses these limitations by integrating predictive modeling through multi-step state predictions. Our method outperforms traditional behavior cloning numerical benchmarks, demonstrating superior robustness to distribution shift and measurement noise both in available data and during execution. Furthermore, we provide theoretical guarantees on the sample complexity and error bounds of our method, offering insights into its convergence properties. 
\end{abstract}

\section{Introduction}

Imitation learning (IL) is a machine learning approach where agents learn to perform tasks by replicating the behavior demonstrated by experts. It has become a widely used method for tackling complex decision-making problems, finding applications in areas such as autonomous driving~\cite{pan2017agile}, robotics~\cite{schaal1999imitation, fang2019survey}, industrial robotics~\cite{liu2022robot}. By leveraging expert demonstrations, IL enables agents to acquire skills efficiently, making it particularly well-suited for problems characterized by high-dimensional decision spaces and intricate dynamics~\cite{osa2018algorithmic}. With the increasing availability of large and diverse datasets, IL is gaining traction as a powerful approach to leveraging data-rich environments; see~\cite{hussein2017imitation} for an extensive review. Yet, IL remains challenging due to the temporal nature of decision-making, where actions may have long-term effects. In long-horizon tasks, small deviations from the expert's trajectory can compound, leading to significant errors and distribution shift~\cite{ross2010efficient}.

A widely used approach in imitation learning is \textit{Behavior Cloning} (BC)~\cite{pomerleau1988alvinn}, which trains a policy to map states to actions using supervised learning on expert demonstrations. While simple and intuitive, BC suffers from compounding errors—small deviations from the expert’s trajectory can accumulate, pushing the agent into unfamiliar states and degrading performance~\cite{ross2010efficient}. This issue stems from \textit{distribution shift}, where the agent encounters states during deployment that were underrepresented or absent in training. Due to the sequential nature of decision-making, even minor inaccuracies can drive the agent into unseen regions of the state space, further amplifying errors.

To mitigate distribution shift and compounding errors, various strategies have been explored. Online approaches like DAgger~\cite{ross2011reduction} iteratively expand the training data by incorporating the agent’s states and querying the expert, while DART~\cite{laskey2017dart} injects stochastic noise into demonstrations to enhance robustness. Though effective, these methods often require additional interactions with expert during training. To alleviate this issue, offline IL methods have been developed. GAIL~\cite{ho2016generative} avoids repeated expert queries by framing imitation learning as an adversarial game, but its training can be unstable and difficult to tune. Alternative approaches, like TASIL~\cite{pfrommer2022tasil}, aim to align higher-order derivatives of the learned policy with the expert’s but rely on access to expert gradients, which can be impractical. Recent works have explored different policy parameterizations to mitigate error accumulation.~\cite{foster2024behavior} proposes training policies with parameter sharing across time steps to mitigate error propagation. Similarly,~\cite{block2024provable} introduces low-level stabilizing controllers to counteract error accumulation induced by the learned policy, demonstrating that policies with short-horizon guarantees can still perform well over longer horizons. While these methods mitigate the effects of distribution shift, their inherently model-free nature limits their ability to anticipate long-term consequences and handle complex dynamics effectively. Without explicit system knowledge, they rely solely on observed data, making them reactive rather than predictive in managing evolving state distributions.

Beyond compounding errors and distribution shift, measurement noise poses a fundamental challenge in imitation learning. Sensor inaccuracies and environmental disturbances create discrepancies between the expert’s trajectory and observed data, degrading learning performance and exacerbating distribution shift.While recent methods such as DIDA~\cite{huang2024dida} and USN~\cite{yu2024usn} introduce techniques to improve robustness against noise, they do so without leveraging a system model, limiting their ability to separate noise from true dynamics. Integrating noise-mitigation strategies with model-based reasoning could provide a more principled way to handle measurement uncertainty, particularly in safety-critical applications where precision is essential.



 
Given the limitations of model-free imitation learning methods, incorporating explicit dynamics models into IL frameworks presents a promising direction. Early work on trajectory matching~\cite{englert2013model} and dynamics-based planning~\cite{hu2022model} improved long-term performance through rollouts with learned models. However, these rollout-based approaches are susceptible to compounding errors over extended horizons, leading to deviations from expert behavior. A predictive framework that continuously refines estimates can mitigate this issue by dynamically adapting to system uncertainties, ensuring more reliable long-term decision-making.



\textit{Contribution:}
This paper introduces a novel model-based imitation learning framework that mitigates compounding errors and measurement noise by integrating predictive modeling with multi-step state and control predictions. Our approach explicitly incorporates system dynamics and leverages learned multi-step predictors to model the closed-loop evolution of state trajectories under the learned policy, rather than relying on rollout-based trajectory optimization. This formulation alleviates computational challenges associated with trajectory alignment and dynamic constraints. Building on this foundation, we propose a surrogate optimization framework that enables a computationally efficient method for minimizing state trajectory discrepancies over finite horizons, enhancing both stability and scalability in imitation learning.


We provide a neural network-based implementation and establish theoretical guarantees, including sample complexity bounds and error convergence. Empirical evaluations on benchmark tasks demonstrate superior performance over existing baselines in both noiseless and noisy settings. These contributions offer a principled and scalable solution for improving imitation learning in complex and realistic scenarios.

\textit{Notation:} We denote a matrix's positive semi-definiteness as $ Q \succeq 0$ and define $ \|x\|_Q^2 = x^\top Q x$. A multivariate Gaussian vector $ x$ with mean $ \mu$ and covariance $ \Sigma$ is written as $ x \sim \mathcal{N}(\mu, \Sigma)$.

  \section{Problem setup and Preliminaries}


Consider a non-autonomous dynamical system given by:
\begin{equation}
    x_{t+1} = f(x_t, u_t),
\end{equation}
where $x_t$, $u_t$ denotes the system's state and control input at time $t$, respectively, and $f$ denotes the known system dynamics. We consider an expert control policy $ \pi^{\exp} $ that generates an expert trajectory $ \{x^{\exp}_t, u^{\exp}_t\}_{t=0}^{T} $, where  
\begin{equation}
    x_{t+1}^{\exp} = f(x_{t}^{\exp}, u_{t}^{\exp}),\ u^{\exp}_t = \pi^{\exp}(x^{\exp}_t).
\end{equation} 
However, the observed dataset is noisy, and we instead have access to the measurements:
\begin{align}
    y^{\exp}_t = x^{\exp}_t + \xi_t,\ v^{\exp}_t = \pi^{\exp}(x^{\exp}_t) + \eta_t,
\end{align}
where $\xi_t$ and $\eta_t$ stand for the measurement noise for states and inputs respectively that are drawn from arbitrary unknown distributions. Our objective is to learn an approximate policy $\hat \pi$ from the available data, $\{y^{\exp}_t, v^{\exp}_t\}_{t=0}^T$, such that the trajectory under $\hat \pi$ closely matches the evolution under the true policy $\pi^{\exp}$. 
Importantly, while the expert policy $\pi^{\exp} $ operates on the true states, the learned policy $\hat{\pi} $ must rely on noisy state measurements during execution. Consequently, the state evolution under $\hat{\pi} $ follows:
\begin{equation}
    \hat{x}_{t+1} = f(x_t, \hat{\pi}(y_t)), \quad \hat y_t  = \hat{x}_t + \xi_t.
\end{equation}
Accordingly, we aim to minimize the worst-case deviation between the trajectories evolving under the expert and learned policy over time:
\begin{align}\label{eq:im-gap}
    \hat{\pi} = \arg \min_{\hat{\pi}} \ &\E {\max_t \|x_t^{\exp} - \hat{x}_t\|}\ \textnormal{s.t. } x^{\exp}_0 = \hat x_0. 
\end{align}

Maximum trajectory discrepancy, as defined in Eq.~\eqref{eq:im-gap}, is a widely used performance metric in control and imitation learning~\cite{pfrommer2022tasil,block2024provable}, favoring learned policies that closely track expert trajectories. However, maintaining a bounded discrepancy over time is challenging due to error propagation and measurement noise. Small deviations from expert trajectories arising from learning inaccuracies and measurement noise can compound over time, resulting in a significant drift.

\section{Proposed Solution}




We begin with describing the standard BC approach and then outline a rollout-based approach that aims to minimize trajectory discrepancy over finite horizons to mitigate compounding errors. We analyze its limitations, focusing on computational challenges, and introduce a surrogate optimization framework that utilizes predictive models to address them. Finally, we discuss its neural network-based implementation and provide theoretical justifications.


\subsection{Behavior Cloning}\label{sec:bc}

BC is a supervised learning approach that learns a policy by imitating expert demonstrations. Given a dataset of expert trajectories with noisy observations, $ \{y_t^{\exp}, v_t^{\exp}\}_{t=0}^{T} $, BC trains a policy $ \pi_\theta$ by minimizing the discrepancy between the predicted and measured expert actions:
\begin{equation}
    \arg \min_{\pi_\theta} \sum_{t=0}^{T}  \| v^{\exp}_t - \pi_\theta(y^{\exp}_t)\|.
\end{equation}
While BC effectively mimics expert behavior in distribution, it suffers from compounding errors due to distribution shift. Additionally, measurement noise in the data further heightens these challenges.   

\subsection{Rollout-Based Imitation Learning}\label{sec:pil}

A natural approach is to learn a policy that directly minimizes Eq.~\eqref{eq:im-gap}. However, optimizing this metric over a long horizon is computationally demanding. Instead, a common approach is to optimize over a shorter finite horizon and re-solve the problem at each time step, providing a more tractable alternative to long-horizon optimization, an idea that has been long-adopted by Model-Predictive Control (MPC)~\cite{kouvaritakis2016model}. This motivates rollout-based imitation learning framework, defined as:

\begin{subequations}
\label{eq:rollout}
    \small
    \begin{align}
        \min_{\pi_{\theta}}\; & \sum_{t=0}^{T-H}\sum_{\tau=1}^{H}  \|\epsilon_{\mathrm{y}, t+\tau | t }\|^2_{Q}+ \|\epsilon_{\mathrm{v}, t+\tau | t }\|^2_{R} \label{eq:naive-loss}, \\
        \textnormal{s.t.}\; 
         & x_{t|t} = y_t^{\exp},\label{eq:rollout-init}\\
         & u_{t+\tau-1|t} = \pi_\theta (x_{t+\tau-1|t}), \label{eq:rollout-pi}\\
         & x_{t+\tau|t} = f(x_{t+\tau-1|t}, u_{t+\tau-1|t}), \label{eq:rollout-dyn}\\
         &\epsilon_{\mathrm{y}, t+\tau | t} = y_{t+\tau}^{\exp} -  x_{t+\tau|t}, \label{eq:rollout-xerr}\\
         &\epsilon_{\mathrm{v}, t+\tau | t} = v_{t+\tau-1}^{\exp} -  u_{t+\tau-1|t}\label{eq:rollout-uerr}.
    \end{align}
\end{subequations}

We denote the predicted state and control input at time $t+\tau $, given information at time $t $, as $x_{t+\tau|t} $ and $u_{t+\tau|t} $, respectively. These predictions are iteratively computed over the planning horizon using the system dynamics~\eqref{eq:rollout-dyn} and the policy~\eqref{eq:rollout-pi}, following a receding horizon approach starting from the current measurement~\eqref{eq:rollout-init}. The objective in \eqref{eq:naive-loss} minimizes two types of errors over a finite horizon $H$: the state error, $\epsilon_{\mathrm{y}, t+\tau | t}$ , measuring the discrepancy between the predicted states and expert states~\eqref{eq:rollout-xerr}, and the control error, $\epsilon_{\mathrm{v}, t+\tau | t}$, capturing the deviation between predicted control inputs and expert controls~\eqref{eq:rollout-uerr}, weighed by $Q, R \succeq 0$ matrices respectively. The choice of $H$ balances computational efficiency and accuracy: shorter horizons simplify optimization by reducing sequential dependencies, while longer horizons improve decision-making at the cost of increased computation.


This formulation explicitly aligns the learned policy with expert trajectories while incorporating system dynamics and finite-horizon planning, making it robust to compounding errors and distribution shifts. However, solving this problem is \textit{computationally demanding} due to the trajectory-level alignment and dynamic constraints. Explicitly minimizing this objective requires access to the gradients of the system dynamics, which may not always be readily available. 

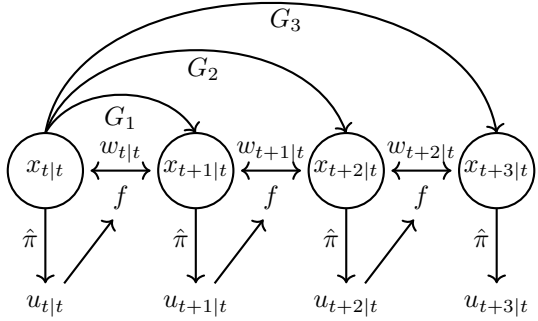
\begin{figure}[t]
\centering
\begin{tikzpicture}

   \draw[thick] (0,0) circle(0.5);
    \node at (0,0) {$x_{t|t}$};

    \draw[thick] (2,0) circle(0.5);
    \node at (2,0) {$x_{t+1|t}$};

    \draw[thick] (4,0) circle(0.5);
    \node at (4,0) {$x_{t+2|t}$};

    \draw[thick] (6,0) circle(0.5);
    \node at (6,0) {$x_{t+3|t}$};

    \draw[thick,->] (0,0.5) to[out=60,in=120] (2,0.5);
    \node at (1, 0.7) {$G_1$};
    \draw[thick,->] (0,0.5) to[out=70,in=110] (4,0.5);
    \node at (2.1, 1.35) {$G_2$};
    \draw[thick,->] (0,0.5) to[out=80,in=100] (6,0.5);
    \node at (3.2, 2.0) {$G_3$};

    \draw[thick,<->] (0.6,0) -- (1.4,0);
    \node at (1.0, 0.25) {$w_{t|t}$};
    \node at (1.0, -0.35) {$f$};
    \draw[thick,<->] (2.6,0) -- (3.4,0);
    \node at (3.0, 0.25) {$w_{t+1|t}$};
    \node at (3.0, -0.35) {$f$};
    \draw[thick,<->] (4.6,0) -- (5.4,0);
    \node at (5.0, 0.25) {$w_{t+2|t}$};
    \node at (5.0, -0.35) {$f$};

    \draw[thick,->] (0,-0.5) -- (0,-1.5);
    \node at (-0.2,-0.9) {$\hat \pi$};
    \node at (0,-1.8) {$u_{t|t}$};
    \draw[thick,->] (0.25,-1.5) -- (0.9,-0.65);

    \draw[thick,->] (2,-0.5) -- (2,-1.5);
    \node at (1.8,-0.9) {$\hat \pi$};
    \node at (2,-1.8) {$u_{t+1|t}$};
    \draw[thick,->] (2.25,-1.5) -- (2.9,-0.65);

    \draw[thick,->] (4,-0.5) -- (4,-1.5);
    \node at (3.8,-0.9) {$\hat \pi$};
    \node at (4,-1.8) {$u_{t+2|t}$};
    \draw[thick,->] (4.25,-1.5) -- (4.9,-0.65);

    \draw[thick,->] (6,-0.5) -- (6,-1.5);
    \node at (5.8,-0.9) {$\hat \pi$};
    \node at (6,-1.8) {$u_{t+3|t}$};


\end{tikzpicture}
\caption{Illustration of the Predictive Imitation Learning~\eqref{eq:mpc-proposed} for horizon $H=3$.}
\label{fig:framework}
\end{figure}

\subsection{Predictive Imitation Learning}



In the following, we present our proposed surrogate optimization framework Predictive Imitation Learning (PIL), designed to address the limitations discussed in Sec.~\ref{sec:pil}. The core idea is to leverage parametrized multi-step predictors $G_\vartheta:=\left\{G_{1,\vartheta},G_{2,\vartheta},\ldots, G_{H, \vartheta}\right\}$  to compute intermediate states across the optimization horizon, bypassing the need to explicitly unroll the dynamics at every time step. Here, $G_{\tau, \vartheta}$ represents $\tau$-step ahead predictors that model the evolution of states under learned policy. These predictors provide an efficient mechanism for approximating state transitions, reducing the computational burden associated with traditional methods. The optimization objective minimizes the discrepancy between the predicted states and inputs and the available expert data over the prediction horizon, ensuring the trajectory remains close to the observed behavior. Crucially, we incorporate a consistency loss term that encodes the system's dynamics into the optimization process, ensuring that the predicted input-state trajectory aligns with the true dynamics $f$. This formulation not only alleviates computational overhead but also enhances trajectory accuracy by effectively incorporating system knowledge into the optimization problem.

\begin{subequations}
\label{eq:mpc-proposed}
    \small
    \begin{align}
        \min_{\pi_{\theta}, G_{\vartheta}}\; & \sum_{t=0}^{T-H}\sum_{\tau=1}^{H} \|\epsilon_{\mathrm{y}, t+\tau | t }\|^2_{Q}+ \|\epsilon_{\mathrm{v}, t+\tau | t }\|^2_{R} + \|w_{t+\tau|t}\|^2_{P} \label{eq:prop-loss}, \\
        \textnormal{s.t.}\; 
         & x_{t|t} = y_t^{\exp},\\
         & u_{t+\tau-1|t} = \pi_\theta (x_{t+\tau-1|t}),\\
         & x_{t+\tau|t} = G_{\tau, \vartheta} (x_{t|t}), \label{eq:pil-msp}\\
         & w_{t+\tau-1|t} = x_{t+\tau|t} - f(x_{t+\tau-1|t}, u_{t+\tau-1|t}),\label{eq:pil-consistency}\\
         &\epsilon_{\mathrm{y}, t+\tau | t} = y_{t+\tau}^{\exp} -  x_{t+\tau|t}, \\
         &\epsilon_{\mathrm{v}, t+\tau | t} = v_{t+\tau-1}^{\exp} -  u_{t+\tau-1|t}.
    \end{align}    
\end{subequations}

The proposed learning framework follows a similar strategy to~\ref{sec:pil}, employing parameterized multi-step predictors to efficiently approximate state transitions and reduce the computational burden typically associated with unrolling system dynamics. Instead of directly enforcing the dynamics constraints in eq.\eqref{eq:rollout-dyn}, our approach leverages multi-step predictors\eqref{eq:pil-msp} and a consistency term~\eqref{eq:pil-consistency} to align predicted trajectories with true system dynamics. By integrating system knowledge through this consistency loss, the method constructs a predicted state trajectory that bridges the expert's trajectory evolution with that of the learned policy. Addressing challenges such as compounding errors and distribution shifts through horizon-aware optimization, this approach effectively captures the foresight of predictive modeling, enhancing generalization across extended trajectories.

\subsection{Implementation}

The proposed algorithm is implemented by parameterizing its components with learnable functions to facilitate efficient training and generalization. We denote the policy parameters by $\theta$ and the parameters of the multi-step predictors by $\vartheta$. The key components and design choices are summarized below:

\begin{itemize}
    \item \textit{Multi-step Predictor Parameterization:} The initial observation $y_t$ is encoded into a latent representation $z_t = \mathrm{Encode}_{\vartheta_1}(y_t)$ using an encoder $\mathrm{Encode}_{\vartheta_1}(\cdot)$, which extracts relevant features from the measurement. This latent representation is shared across a set of multi-step predictors $\{G_{\tau, \vartheta_2}\}$, promoting consistency and reducing redundancy in learning. Each predictor takes $z_t$ as input and produces a future state prediction $x_{t+\tau|t} = G_{\tau, \vartheta_2}(z_t)$. The predictors are thus jointly parameterized by the encoder and the predictor components, i.e., $\vartheta_1$ and $\vartheta_2$.
    
    \item \textit{Policy Parameterization:} Control inputs $u_{t+\tau|t}$ are generated by a paremetrized policy $\pi_{\theta}(\cdot)$, which maps predicted states to control actions: $u_{t+\tau|t} = \pi_{\theta}(x_{t+\tau|t})$.
    
    \item \textit{Training Objective:} The parameters $(\theta, \vartheta_1, \vartheta_2)$ are trained jointly by minimizing the loss defined in Eq.~\eqref{eq:prop-loss}, which encourages alignment between predicted trajectories and expert demonstrations while penalizing residual dynamics errors.

\end{itemize}

This modular design allows scalable and efficient training, making it well-suited for high-dimensional and complex control tasks. The encoder can be interpreted as a form of filtering or denoising that produces a compact and informative representation of the measurement $y_t$ for downstream decision-making.

  \section{Performance Guarantees}

In this section, we provide a few refined theoretical results for the special case of LTI systems, where $f(x_t, u_t) = Ax_t + Bu_t$ ($x_t \in \R^n$, $u_t \in \R^m$) subject to Gaussian distrubed measurement noise: $\xi_t\sim\mathcal{N}(0, \Sigma_\xi),\ \eta_t\sim\mathcal{N}(0, \Sigma_\eta)$. The control policies are restricted to linear feedback policies, i.e. $\pi^{\exp}(x_t) = \Ks x_t$ and $\hat{\pi}(x_t) = \hK x_t$. For a detailed explanation of the setup, please refer to~\cite{onlinereport}.

We make the following assumption on data coverage.

\begin{assumption}\label{assum:data}
  There exists a constant $\phi_x$, such that
  \begin{equation}
    \sum_{t=0}^{T-H} x_t^{\exp} x_t^{\exp,\top} \succeq \phi_x (T-H+1) I.        
  \end{equation}
\end{assumption}

The following theorem characterizes the error of the proposed MS-PIL controller in linear systems.

\begin{theorem}[learning error, \textit{sketched}]\label{thm:error}
  Under \Cref{assum:data}, for any $T \geq O(n \log (1/\delta))$, with probability at least $1-\delta$ we have
  \begin{equation*}
    \norm{\hKMS-\Ks} = \frac{\kappa_1}{\sqrt{T-H+1}} + \kappa_2 \norm{\varSigma_{\xi}},
  \end{equation*}
  where $\kappa_1, \kappa_2$ are system-specific constants.
\end{theorem}

\begin{proof}[Proof Sketch]
We begin by we quantifying the estimation error of $\hK$ (\ifbool{arxiv}{ Lemma~\ref{lem:opt}}{Appendix A Lemma 4}) which is of the form of a ratio of two terms. Then, the proof follows these key steps:
\begin{enumerate}
    \item \textit{Bounding the numerator:} The terms in numerator consists of noise-related components. Applying concentration inequalities, we show that its' norm scales as $O(\sqrt{T-H+1}) + O(\psi_\xi)$.

    \item \textit{Bounding the denominator:} The denominator is shown to be well-conditioned with high probability using matrix concentration bounds. This ensures its inverse scales as $O(1/(T-H+1))$.

    \item \textit{Combining the bounds:} Multiplying the bounds on the numerator and denominator, we obtain the sample complexity rate.
\end{enumerate}
Due to space constraints, the detailed proof is included in~\ifbool{arxiv}{Appendix B}{Appendix B of the online report~\cite{onlinereport}}.
\end{proof}

The above theorem decomposes the learning error into two terms: the first term is a time-decaying term that illustrate the advantage of multi-step prediction with longer horizon $T$, which provides more data, and the second term indicates that the quality of the controller is impacted by the accuracy of state observations.

We now proceed to explicitly compare the policies obtained from different methods. However, such comparison turns out to be hard if multi-step rolling-out is involved. Hence for the sake of clarity, we temporarily restrict our scope to the special case with 1 step prediction, i.e., $H = 1$. In this case,  the following theorem compares the PIL controller $\hKMS$ against behavior cloning linear controller $\hKBC$, in terms of their ``typical'' distances from the expert controller $\Ks$.

\begin{theorem}[comparison of $\hKMS$ against $\hKBC$, \textit{sketched}]\label{thm:comparison}
  Under \Cref{assum:data}, suppose $H=1$ and $G_1$ is fixed to be $G_1 = A + BK$, the distances $\norm{\hKMS - \Ks}$ and $\norm{\hKBC - \Ks}$ shall be upper bounded by
  \begin{align*}
    \norm{\hKMS - \Ks} &\leq \alpha + \beta \norm{\omegaPIL},\\
    \norm{\hKBC - \Ks} &\leq \alpha + \beta \norm{\omegaBC},
  \end{align*}
  where $\alpha, \beta$ are system-specific constants. Moreover, when
  \begin{equation*}
    \tr(\varSigma_{\xi}) \leq C \tr(\varSigma_{\eta})
  \end{equation*}
  for some system-specific constant $C$, we have
  \begin{equation*}
    \E{\norm{\omegaPIL}} \leq \E{\norm{\omegaBC}}.  
  \end{equation*}
\end{theorem}

\begin{proof}[Proof Sketch]
For $H=1$, the learned controllers minimize their respective loss functions, where $\hKMS$ incorporates both state transitions and input observations, while $\hKBC$ relies only on input observations. Taking the gradient and solving for the optimal controllers, both solutions can be expressed in terms of system matrices and noise terms.

The error in $\hK$ and $\hKBC$ is determined by these noise terms, denoted as $\omegaPIL$ and $\omegaBC$. Using matrix concentration bounds, we establish upper bounds on their norms, showing that both errors are bounded by a system-dependent constant plus a term proportional to the respective noise magnitude.

Finally, under the assumption that state observation noise is relatively small compared to input noise, the expected magnitude of $\omegaPIL$ is shown to be smaller than that of $\omegaBC$. This implies that $\hKMS$ achieves a lower expected error than $\hKBC$, highlighting the advantage of incorporating state transition information. Due to the space limit, the detailed proof is provided in~\ifbool{arxiv}{Appendix C}{Appendix C of the online report~\cite{onlinereport}}.
\end{proof}

While our theoretical analysis focuses on bounding the policy estimation error $ \|\hKMS - \Ks\| $, our primary evaluation metric in \eqref{eq:im-gap} is defined in terms of state trajectory discrepancies. This connection is particularly straightforward in linear systems with expert policies of known parameterization, where the state evolution is explicitly governed by the policy through the system dynamics. In such cases, a lower policy estimation error directly translates to reduced deviation between the learned and expert trajectories. Thus, our theoretical results provide guarantees that align with and support improvements in the imitation gap metric.


  \section{Experiments}
This section presents numerical results comparing PIL, rollout-based method and BC across various settings. The implementation for the experiments is available at~\url{https://github.com/haldunbalim/PredictiveIL}. For detailed experimental setups, see~\ifbool{arxiv}{Appendix~\ref{ssec:details}}{Appendix E in the online version~\cite{onlinereport}}.

\subsection{Linear System with Linear Expert}
In this section we consider the following linear system:
\begin{align}\label{eq:example-lin}
    A = \begin{bmatrix}
        0.95 & 0.05 \\ 0 & 0.95
    \end{bmatrix},\ B = \begin{bmatrix}
        0 \\ 0.05
    \end{bmatrix},\\
    \xi_t\sim\mathcal{N}(0, \Sigma_\xi),\ \eta_t\sim\mathcal{N}(0, \Sigma_\eta)\notag
\end{align}


The expert policy is the optimal policy $u = K^* x$ for minimizing quadratic cost defined by state cost matrix $I$ and an input cost matrix $10^{-2}I $. The expert generates $50$ training trajectories with $x^{\exp}_0 \sim \mathcal{N}(0, I)$, each spanning $100$ time steps. We compare BC with our proposed approach for different state and action measurement noise levels. The expert’s state evolution follows the transition matrix $A + BK^* $, making the $\tau $-step-ahead predictors of the form $(A + BK^*)^\tau $. In our method, these predictors are estimated via least squares and used to compute the feedback gain through our proposed objective. Note that, since we use fixed multi-step predictors, the loss due to state errors remains constant. We estimate the maximum trajectory discrepancy~\eqref{eq:im-gap} for both methods using $1000$ new test trajectories. As shown in Figure~\ref{fig:lin-sys}, our approach achieves superior long-horizon performance under high state noise and demonstrates a slight advantage over BC in high input noise settings.




\begin{figure}[t]
\vskip 0.2in
\begin{center}
\centerline{\includegraphics[width=\columnwidth]{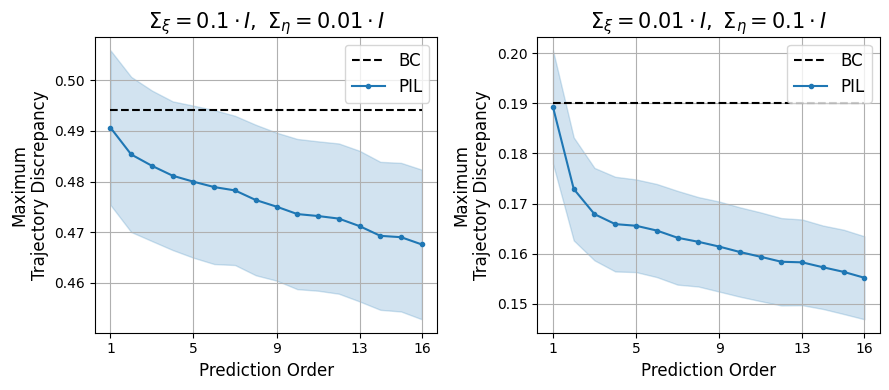}}
\caption{Maximum trajectory discrepancy~\eqref{eq:im-gap} relative to 
$\hKBC$, for varying prediction orders under high state measurement noise (left) and high action measurement noise (right). Results are averaged over $100$ seeds, with the shaded area indicating half standard deviation.}
    \label{fig:lin-sys}
\end{center}
\vskip -0.2in
\end{figure}

\begin{figure}[t]
\vskip 0.2in
\begin{center}
\centerline{\includegraphics[width=0.8\columnwidth]{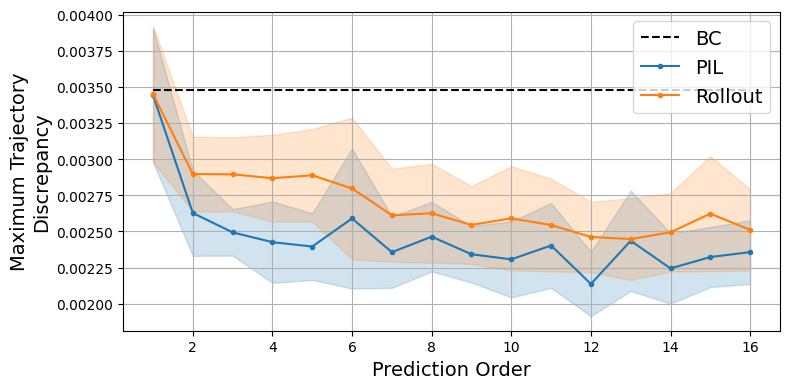}}
\caption{Maximum trajectory discrepancy~\eqref{eq:im-gap} for BC, Rollout-based approach and PIL, for varying prediction orders. Results are averaged over $50$ seeds, with the shaded area indicating half standard deviation.}
    \label{fig:pred-order}
\end{center}
\vskip -0.2in
\end{figure}

\subsection{Linear System with Nonlinear Expert}

We analyze the linear system in Eq.~\ref{eq:example-lin}, driven by a randomly initialized multi-layer perceptron. We aim to track the state evolution using both PIL and rollout-based strategies. We evaluate the maximum trajectory discrepancy across prediction horizons using $100$ test trajectories. The experiment is repeated with $10$ different seeds, and the average results are presented in Figure~\ref{fig:pred-order}, with half-standard deviations indicated. PIL consistently outperforms the other methods across all prediction orders. Furthermore, increasing the prediction order up to a threshold reduces trajectory discrepancy.

\begin{figure*}[t]
    \vskip 0.2in
    \begin{center}
        
    \includegraphics[width=2\columnwidth]{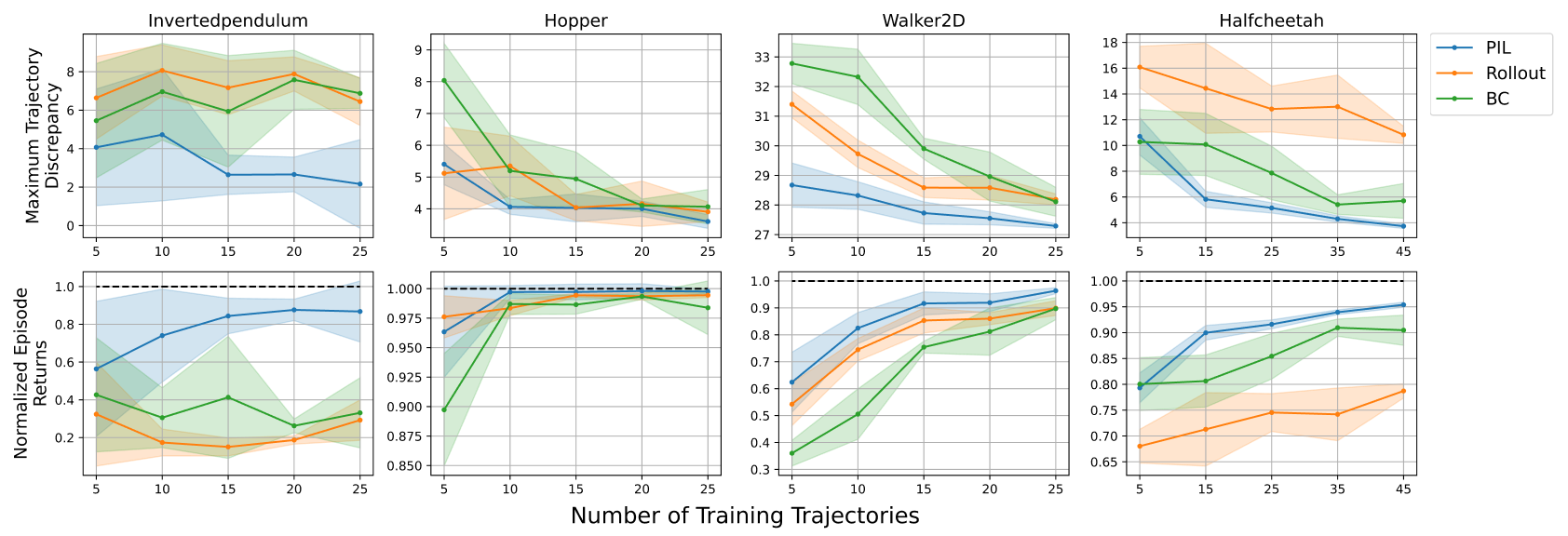}\\    
    \caption{Maximum trajectory discrepancy~\eqref{eq:im-gap} and episode returns (normalized by expert returns) for BC, PIL, and rollout-based approach across different training trajectory counts. Results are averaged over 5 seeds, with shaded regions indicating half a standard deviation.}
     \label{fig:mujoco}

    \end{center}
    \vskip -0.2in
\end{figure*}

\subsection{Influence of Dynamics Gradient}

In this experiment, we assess the significance of dynamics gradients. For this, we compare imitation learning algorithms on an inverted pendulum task with and without measurement noise using $50$ training trajectories of length $100$. PIL and rollout-based method are trained both with and without access to dynamics gradients. Maximum trajectory discrepancy is computed over $1000$ test trajectories, averaged across $5$ seeds (Table~\ref{tab:pendulum}). The proposed method consistently outperforms baselines, showing greater robustness to noise. Gradient-based methods perform slightly better in the noiseless setting, but even without gradients, the proposed approach remains competitive. Rollout-based method improves over BC but is surpassed by PIL. Noise increases discrepancy for all methods, yet the proposed approach exhibits the highest robustness.

\begin{table}[h]
\caption{Maximum trajectory discrepancy for BC, rollout-based, and PIL methods on the Pendulum task. Rollout-based method and PIL are trained both with and without access to the dynamics function's gradients.}
\vskip 0.15in
\begin{center}
\begin{small}
\begin{sc}
\label{tab:pendulum}
\begin{tabular}{lcc}
\toprule
             & No Noise       & With Noise     \\
\midrule
BC           & 0.192          & 0.262          \\
Rollout           & 0.188          & 0.231          \\
Rollout (No Gradient)           &  \underline{0.179}       & 0.234          \\
PIL     & \textbf{0.167} & {\underline {0.213}}    \\
PIL (No Gradient) & {\underline{0.179}}    & \textbf{0.203}\\
\bottomrule
\end{tabular}
\end{sc}
\end{small}
\end{center}
\vskip -0.1in
\end{table}

\subsection{MuJoCo Experiments}

Finally, we evaluate BC, rollout-based, and PIL on several standard continuous control tasks in MuJoCo. The environments considered, along with their respective (state, action) dimensions, are: \texttt{InvertedPendulum-v4} $(4,1)$, \texttt{Hopper-v4} $(11,3)$, \texttt{Walker2d-v4} $(17,6)$, and \texttt{HalfCheetah-v4} $(17,6)$. We provide custom implementations of these environments and utilize the \textit{implicitfast} integrator from the MuJoCo toolbox. Expert demonstrations are generated using RL-trained agents. We collect trajectories of fixed length $300$ and evaluate performance across varying numbers of training trajectories for each task. Additional details can be found in~\ifbool{arxiv}{Appendix~\ref{ssec:details}}{Appendix E of the online version~\cite{onlinereport}}. All results are averaged over $5$ random seeds. In Figure~\ref{fig:mujoco}, we report average episode returns and trajectory discrepancy. As shown, PIL outperforms other methods almost in all cases. 

Our experimental results demonstrate that PIL consistently outperforms both rollout-based approach and BC across various settings, including linear and nonlinear systems, as well as standard MuJoCo benchmarks. Notably, PIL exhibits superior long-horizon performance and robustness to noise, highlighting the effectiveness of our predictive algorithm in imitation learning.
  
\section{Conclusion}
We introduced a model-based imitation learning framework that addresses compounding errors, distribution shift, and measurement noise by integrating predictive modeling with consistency constraints derived from known system dynamics. Unlike traditional rollout-based methods, our approach employs multi-step predictors within a surrogate optimization framework, enabling robust long-horizon performance. We provide finite-sample guarantees on policy estimation error under noisy observations for the special case of linear systems, and validate our method across linear, nonlinear, and high-dimensional continuous control settings, where it consistently outperforms baseline approaches. Future directions include exploring more expressive architectural choices and incorporating learned dynamics models to enhance adaptability and scalability in complex, uncertain environments.

  \bibliographystyle{IEEEtran}
\bibliography{ref}


  \newpage
  \onecolumn
  
  \ifbool{arxiv}{\appendix
  \subsection{Setup and Algorithm for Linear Systems}\label{sec:linear-setup}

In this section, we first restate the problem setup and the proposed algorithms for linear time-invariant (LTI) systems, and then provide closed-form characterizations of these solutions.

\boldtitle{Linear System Setup.} Consider an LTI system specified by
\begin{equation*}
  x_{t+1} = f(x_t, u_t) = Ax_t + Bu_t,\quad
  x_0 = \bar{x}_0,
\end{equation*}
where $x_t \in \R^n$, $u_t \in \R^m$ denote the state and control input at time $t$, and $A \in \R^{n \times n}$, $B \in \R^{n \times m}$ are dynamical matrices. We consider an expert control policy $\pi^{\exp}(x_t) = \Ks x_t$ that generates an expert trajectory $ \{x^{\exp}_t, u^{\exp}_t\}_{t=0}^{T} $, where  
\begin{equation}
    x_{t+1}^{\exp} = Ax_{t}^{\exp} + Bu_{t}^{\exp},\quad
    u^{\exp}_t = \Ks x^{\exp}_t.
\end{equation} 
However, the observed dataset is noisy, and we instead have access to the measurements:
\begin{align}
    y^{\exp}_t &= x^{\exp}_t + \xi_t,\quad \xi_t\sim\mathcal{N}(0, \varSigma_\xi),\\
    v^{\exp}_t &= \Ks x^{\exp}_t + \eta_t,\quad \eta_t \sim\mathcal{N}(0, \varSigma_\eta), \notag
\end{align}
where $\xi_t$ and $\eta_t$ stand for the measurement noise for states and inputs respectively. Our objective is to learn an approximate linear controller $u_t = \hat{pi}(x_t) = \hK x_t$ from the available data, $\{y^{\exp}_t, v^{\exp}_t\}_{t=0}^T$, such that the trajectory under $\hat \pi$ ($\hK$) closely matches the evolution under the true policy $\pi^{\exp}$ ($\Ks$).

For notational simplicity, the superscript ``exp'' will be omitted in the following sections when the context is clear.

\boldtitle{Proposed Algorithms.} We will make heavy use of the $\tau$-step transition matrix $x_{t+\tau} = G^{\star}_{\tau} x_t$, where the ground truth is given by $G^{\star}_{\tau} = (A+BK)^{\tau}$.

The proposed algorithm now becomes the following:
\begin{equation}
  \min_{K, G_{1:H}}\quad \sum_{t=0}^{T-H} \sum_{\tau=1}^{H} \norm*{y_{t+\tau} - G_{\tau} y_t}_Q^2 + \norm*{v_{t+\tau-1} - K G_{\tau-1} y_t}_R^2 + \norm*{G_{\tau} y_t - (A+BK) G_{\tau-1} y_t}_P^2.
\end{equation}

\boldtitle{Characterization of the Solutions.} a

\begin{lemma}\label{lem:opt}
    The optimal solution $\hK $ is given by
    \begin{equation}\begin{split}
        \hK = \Ks + (R + B^{\top} P B)^{-1} \left( \sum_{t,\tau} R(\eta_t - \Ks G_{\tau-1} \xi_t) (G_{\tau-1} y_t)^{\top} \right) \left( \sum_{t,\tau} (G_{\tau-1} y_t) (G_{\tau-1} y_t)^{\top} \right)^{-1}.
    \end{split}\end{equation}
\end{lemma}

\begin{proof}
    Taking the gradient of the loss function with respect to $K$ and setting it to zero, we get
    \begin{equation}\begin{split}
        0 = {}& \sum_{t=0}^{T-H}\sum_{\tau=1}^H R\b{v_{t+\tau-1} - \hK G_{\tau-1}y_t}\b{G_{\tau-1}y_t}^\top + B^{\top}P\b{G_{\tau}y_t - \b{A+B\hK}G_{\tau-1}y_t}\b{G_{\tau-1}y_t}^\top\\
        = {}& -(R+B^{\top}PB)\hK \sum_{t=0}^{T-H}\sum_{\tau=1}^H \b{G_{\tau-1}y_t}\b{G_{\tau-1}y_t}^\top\\
        {}& \hspace{2em} + \sum_{t=0}^{T-H}\sum_{\tau=1}^H \b{Rv_{t+\tau-1}+B^\top PG_{\tau}y_t-B^\top PAG_{\tau-1}y_t}\b{G_{\tau-1}y_t}^\top
    \end{split}\end{equation}
    which is equivalent to 
    \begin{equation}\begin{split}
        \hK = {}& \b{R+B^\top PB}^{-1}\b{\sum_{t=0}^{T-H}\sum_{\tau=1}^H \b{Rv_{t+\tau-1}+B^\top PG_{\tau}y_t-B^\top PAG_{\tau-1}y_t}\b{G_{\tau-1}y_t}^\top}\\
        {}& \hspace{2em}\b{\sum_{t=0}^{T-H}\sum_{\tau=1}^H \b{G_{\tau-1}y_t}\b{G_{\tau-1}y_t}^\top}^{-1}\\
        = {}& \b{R+B^\top PB}^{-1}\b{\sum_{t=0}^{T-H}\sum_{\tau=1}^H \b{Rv_{t+\tau-1}+B^\top PB\Ks G_{\tau-1}y_t}\b{G_{\tau-1}y_t}^\top}\b{\sum_{t=0}^{T-H}\sum_{\tau=1}^H \b{G_{\tau-1}y_t}\b{G_{\tau-1}y_t}^\top}^{-1}\\
        = {}& \Ks + (R + B^{\top} P B)^{-1} \left(\sum_{t=0}^{T-H}\sum_{\tau=1}^H R(\eta_t - \Ks G_{\tau-1} \xi_t) (G_{\tau-1} y_t)^{\top} \right) \left(\sum_{t=0}^{T-H}\sum_{\tau=1}^H (G_{\tau-1} y_t) (G_{\tau-1} y_t)^{\top} \right)^{-1}.
\end{split}\end{equation}
\end{proof}

\subsection{Sample Complexity Analysis for Linear Systems}

In this section, we show a sample complexity bound for the proposed algorithm in linear systems.

We define the following notations
\begin{equation}\begin{split}
    {}& \rho_A < 1,\quad \psi_A = \norm*{A},\quad \psi_B = \norm*{B}, \quad \psi_{\xi} = \norm*{\varSigma_{\xi}}, \quad \psi_\eta = \norm*{\varSigma_\eta}, \quad \psi_\xi = \norm*{\varSigma_\xi},\quad \phi_\xi = \sigma_{\min}\b{\varSigma_\xi},\\
{}& \psi_K = \norm*{\Ks}, \quad \psi_u = \max\{\norm*{u_t}, \norm*{x_0}\}, \quad \psi_G = \max\{\sum_{\tau=1}^H \norm*{G_{\tau}}, \sum_{\tau=1}^H \norm*{G_{\tau}}^2\}.
\end{split}\end{equation}
Suppose $x_t\in\mathbb{R}^{n}, u_t\in\mathbb{R}^m$. 

The following proofs work under \Cref{assum:data}.

\begin{lemma}
    Suppose $T$ satisfies
    \begin{equation}\begin{split}
        T \gtrsim \frac{\psi_\xi}{\phi_\xi^2} n\log\b{\frac{15\b{1+\psi_x^2}^{\frac{1}{2}}}{\delta}}\label{eq:err1}.
    \end{split}\end{equation}
    Then with probability at least $1-\delta$
    \begin{equation}\begin{split}
        \norm*{\hK-\Ks} = {}& \kappa_1 \frac{1}{\sqrt{T-H+1}} + \kappa_2 \psi_\xi.
    \end{split}\end{equation}
    Here
    \begin{equation}\begin{split}
        {}& \kappa_1 = {\frac{2\psi_G\norm*{\b{R+B^\top PB}^{-1}R}}{\phi_x+\phi_\xi}}\b{C_1+C_2}, \quad \kappa_2 = {\frac{4\psi_K\psi_G\norm*{\b{R+B^\top PB}^{-1}R}}{\phi_x+\phi_\xi}},\\
        {}& C_1 = \sqrt{8\max\{m,n\}\b{1+\psi_x}\b{\psi_\eta+\psi_K^2\psi_\xi\psi_G}\log\b{\frac{15H(1+\psi_x^2)^{\frac{1}{2}}}{\delta}}},\\
        {}& C_2 = \sqrt{8\max\{m,n\}\psi_{\eta}(1+2\psi_{\xi})\log\b{\frac{30\b{1+\psi_\xi}^{\frac{1}{2}}}{\delta}}}, \quad \psi_x = \b{\psi_u\psi_A\psi_B\frac{3-2\rho_A}{1-\rho_A}}^2.
    \end{split}\end{equation}
\end{lemma}

\begin{proof}
    By Lemma \ref{lem:opt}, we know that
    \begin{equation}\begin{split}
        \norm*{\hK - \Ks} = \norm*{\b{R+B^\top PB}^{-1}R \b{\sum_{\tau=1}^H \sum_{t=0}^{T-H} \b{\eta_t - \Ks G_{\tau-1}\xi_t} \b{G_{\tau-1}y_t}^\top}\b{\sum_{\tau=1}^H \sum_{t=0}^{T-H} \b{G_{\tau-1}y_t} \b{G_{\tau-1}y_t}^\top}^{-1}}.
    \end{split}\end{equation}

    To bound the error, we first consider the middle term, which can be decomposed as follows
    \begin{equation}\begin{split}\label{eq:sample1}
        {}& \sum_{\tau=1}^H \sum_{t=0}^{T-H} \b{\eta_t - \Ks G_{\tau-1}\xi_t} y_t^\top G_{\tau-1}^\top\\
        = {}& \sum_{\tau=1}^H \sum_{t=0}^{T-H} \b{\eta_t - \Ks G_{\tau-1}\xi_t} \b{x_t+\xi_t}^\top G_{\tau-1}^\top\\
        = {}& \sum_{\tau=1}^H \sum_{t=0}^{T-H} \b{\eta_t - \Ks G_{\tau-1}\xi_t} x_t^\top G_{\tau-1}^\top + \sum_{\tau=1}^H \sum_{t=0}^{T-H} \eta_t \xi_t^\top G_{\tau-1}^\top + \Ks \sum_{\tau=1}^H \sum_{t=0}^{T-H} G_{\tau-1}\xi_t \xi_t^\top G_{\tau-1}^\top.
    \end{split}\end{equation}

    Consider the first term in Equation \ref{eq:sample1}. For any fixed $\tau\in[1,H]$,
    \begin{equation}\begin{split}
        \norm*{\sum_{t=0}^{T-H} \b{\eta_t - \Ks G_{\tau-1}\xi_t} x_t^\top}\leq {}& \norm*{\sum_{t=0}^{T-H} \b{\eta_t - \Ks G_{\tau-1}\xi_t} x_t^\top \b{\overline{\varSigma}_x}^{-\frac{1}{2}}}\norm*{\overline{\varSigma}_x}^{\frac{1}{2}}.
    \end{split}\end{equation}
    Here we let $\overline{\varSigma}_x = \sum_{t=0}^{T-H} x_tx_t^\top + (T-H+1)I$.
    Notice that the noise $\eta_t - \Ks G_{\tau-1}\xi_t$ is independent of the state $x_t$ and that the covariance of the noise is $\varSigma_{\eta}+\Ks G_{\tau-1}\varSigma_{\xi}G_{\tau-1}^\top \Ks {}^\top$. Hence we apply Proposition 8.2 with $R^2 = \psi_{\eta}+\psi_K^2\psi_{\xi}\norm*{G_{\tau-1}}^2$ and get the following with probability at least $1-\delta/(3H)$
    \begin{equation}\begin{split}
        \norm*{\sum_{t=0}^{T-H} \b{\eta_t - \Ks G_{\tau-1}\xi_t} x_t^\top \b{\overline{\varSigma}_x}^{-\frac{1}{2}}} \leq {}& \sqrt{8\b{\psi_\eta+\psi_K^2\psi_\xi\norm*{G_{\tau-1}}^2}\log\b{\frac{5^m\b{1+\psi_x}^{\frac{n}{2}}}{\delta/(3H)}}}\\
        \leq {}& 2\sqrt{2}\sqrt{\max\{m,n\}\b{\psi_\eta+\psi_K^2\psi_\xi\norm*{G_{\tau-1}}^2}\log\b{\frac{15H(1+\psi_x)^{\frac{1}{2}}}{\delta}}}.
    \end{split}\end{equation}
    where $\psi_x = \b{\psi_u\psi_A\psi_B\frac{3-2\rho_A}{1-\rho_A}}^2$. Therefore
    \begin{equation}\begin{split}
        {}& \norm*{\sum_{t=0}^{T-H} \b{\eta_t - \Ks G_{\tau-1}\xi_t} x_t^\top}\leq \norm*{\sum_{t=0}^{T-H} \b{\eta_t - \Ks G_{\tau-1}\xi_t} x_t^\top \b{\overline{\varSigma}_x}^{-\frac{1}{2}}}\norm*{\overline{\varSigma}_x}^{\frac{1}{2}}\\
        \leq {}& \sqrt{8\max\{m,n\}\b{\psi_\eta+\psi_K^2\psi_\xi\norm*{G_{\tau-1}}^2}\log\b{\frac{15H(1+\psi_x^2)^{\frac{1}{2}}}{\delta}}}\cdot \sqrt{\b{1+\psi_x}(T-H+1)}\\
        = {}& \underbrace{\sqrt{8\max\{m,n\}\b{1+\psi_x}\b{\psi_\eta+\psi_K^2\psi_\xi\psi_G}\log\b{\frac{15H(1+\psi_x^2)^{\frac{1}{2}}}{\delta}}}}_{C_1}\sqrt{T-H+1}.
    \end{split}\end{equation}
    Here the second last inequality is by Lemma \ref{lem:cov}. By a union bound, the above inequality holds for all $\tau\in[1,H]$ with probability at least $1-\delta/3$

    Now consider the second and third term in Equation \ref{eq:sample1}. By proposition 3.1 of Sasha, with probability at least $1-\delta/6$, 
    \begin{equation}\begin{split}
        \norm*{\sum_{t=0}^{T-H} \eta_t\xi_t^\top}
        \leq {}& \norm*{\sum_{t=0}^{T-H} \eta_t\xi_t^\top\b{\sum_{t=0}^{T-H}\xi_t\xi_t^{\top}+(T-H+1)I}^{-\frac{1}{2}}}\norm*{\b{\sum_{t=0}^{T-H}\xi_t\xi_t^{\top}+(T-H+1)I}^{\frac{1}{2}}}\\
        \leq {}& \sqrt{8\psi_{\eta}\log\b{\frac{5^m\b{1+\psi_\xi}^{\frac{n}{2}}}{\delta/6}}}\norm*{\b{\sum_{t=0}^{T-H}\xi_t\xi_t^{\top}+(T-H+1)I}^{\frac{1}{2}}}\\
    \end{split}\end{equation}
    By Equation (46) of Sasha, with probability at least $1-\delta/6$,
    \begin{equation}\begin{split}\label{eq:sample2}
        \norm*{\sum_{t=0}^{T-H} \xi_t \xi_t^\top} \leq 2\psi_{\xi} (T-H+1). 
    \end{split}\end{equation}
    Then by a union bound over the above two inequalities, the following holds with probability at least $1-\delta/3$,
    \begin{equation}\begin{split}\label{eq:sample3}
        \norm*{\sum_{t=0}^{T-H} \eta_t\xi_t^\top}
        \leq {}& \underbrace{\sqrt{8\max\{m,n\}\psi_{\eta}(1+2\psi_{\xi})\log\b{\frac{5\b{1+\psi_\xi}^{\frac{1}{2}}}{\delta/6}}}}_{C_2}\sqrt{T-H+1}.
    \end{split}\end{equation}

    Therefore, with probability at least $1-2\delta/3$,
    \begin{equation}\begin{split}
        {}& \norm*{\sum_{\tau=1}^H \sum_{t=0}^{T-H} \b{\eta_t - \Ks G_{\tau-1}\xi_t} y_t^\top G_{\tau-1}^\top}\\
        = {}& \sum_{\tau=1}^H \b{\norm*{ \sum_{t=0}^{T-H} \b{\eta_t - \Ks G_{\tau-1}\xi_t} x_t^\top G_{\tau-1}^\top} + \norm*{\sum_{t=0}^{T-H} \eta_t \xi_t^\top G_{\tau-1}^\top} + \norm*{\Ks \sum_{t=0}^{T-H} G_{\tau-1}\xi_t \xi_t^\top G_{\tau-1}^\top}}\\
        \leq {}& \sum_{\tau=1}^H \norm*{G_{\tau-1}} \b{\b{C_1+C_2}\sqrt{T-H+1} + 2\psi_\xi\psi_K\norm*{G_{\tau-1}}(T-H+1)}\\
        \leq {}& \psi_G\b{C_1+C_2}\sqrt{T-H+1} + 2\psi_\xi\psi_K\psi_G(T-H+1)\\
    \end{split}\end{equation}

    Finally, by Lemma \ref{lem:cov} and Equation \ref{eq:err1}, with probability at least $1-\delta/3$, we have
    \begin{equation}\begin{split}\label{eq:sample4}
        \sum_{\tau=1}^H \sum_{t=0}^{T-H} \b{G_{\tau-1}y_t} \b{G_{\tau-1}y_t}^\top \succeq \sum_{t=0}^{T-H} y_ty_t^\top \succeq \frac{\phi_x+\phi_\xi}{2}(T-H+1)I.
    \end{split}\end{equation}

    Combining the Equation \ref{eq:sample3} and \ref{eq:sample4} with a union bound and we get the following with probability at least $1-\delta$
    \begin{equation}\begin{split}
        \norm*{\hK -\Ks} = {}& \norm*{\b{R+B^\top PB}^{-1}R \b{\sum_{\tau=1}^H \sum_{t=0}^{T-H} \b{\eta_t - \Ks G_{\tau-1}\xi_t} \b{G_{\tau-1}y_t}^\top}\b{\sum_{\tau=1}^H \sum_{t=0}^{T-H} \b{G_{\tau-1}y_t} \b{G_{\tau-1}y_t}^\top}^{-1}}\\
        \leq {}& \norm*{\b{R+B^\top PB}^{-1}R} \norm*{\b{\sum_{\tau=1}^H \sum_{t=0}^{T-H} \b{\eta_t - \Ks G_{\tau-1}\xi_t} \b{G_{\tau-1}y_t}^\top}}\norm*{\b{\sum_{\tau=1}^H \sum_{t=0}^{T-H} \b{G_{\tau-1}y_t} \b{G_{\tau-1}y_t}^\top}^{-1}}\\
        \leq {}& \frac{2\norm*{\b{R+B^\top PB}^{-1}R}}{\phi_x+\phi_\xi} \b{\psi_G\b{C_1+C_2}\frac{1}{\sqrt{T-H+1}} + 2\psi_\xi\psi_K\psi_G}\\
        \leq {}& \underbrace{\frac{2\psi_G(C_1+C_2)\norm*{\b{R+B^\top PB}^{-1}R}}{\phi_x+\phi_\xi}}_{\kappa_1} \frac{1}{\sqrt{T-H+1}} + \underbrace{\frac{4\psi_K\psi_G\norm*{\b{R+B^\top PB}^{-1}R}}{\phi_x+\phi_\xi}}_{\kappa_2}\psi_\xi.
    \end{split}\end{equation}
\end{proof}

\begin{lemma}
    Suppose $T$ satisfies
    \begin{align}
        {}& T-H+1 \gtrsim \frac{\psi_\xi}{\phi_\xi^2} n\log\b{\frac{5\b{1+\psi_x^2}^{\frac{1}{2}}}{\delta}}\label{eq:cov1}.
    \end{align}
    Then the following holds with probability at least $1-\delta$
    \begin{equation}\begin{split}
        \sum_{t=0}^{T-H}y_ty_t^\top \succeq \frac{\phi_x+\phi_\xi}{2}(T-H+1)I.
    \end{split}\end{equation}
\end{lemma}

\begin{proof}
    By the definition of $y_t$, 
    \begin{equation}\begin{split}
        \sum_{t=0}^{T-H}y_ty_t^\top = {}& \sum_{t=0}^{T-H}\b{x_t+\xi_t}\b{x_t + \xi_t}^\top = \sum_{t=0}^{T-H} x_tx_t^\top + \sum_{t=0}^{T-H} \xi_t\xi_t^\top + \sum_{t=0}^{T-H} \xi_tx_t^\top + \sum_{t=0}^{T-H} x_t\xi_t^\top.
    \end{split}\end{equation}
    
    We first upper bound the noise term $\sum_{t=0}^{T-H} \xi_tx_t^\top + \sum_{t=0}^{T-H} x_t\xi_t^\top$ using Proposition 3.1 from Sasha. 
    With the following upper bound on the state covariance (Lemma \ref{lem:cov})
    \begin{equation}\begin{split}
        \sum_{t=0}^{T-H} x_t x_t^\top \preceq \psi_x(T+1)I, \quad \psi_x = \b{\psi_u\psi_A\psi_B\frac{3-2\rho_A}{1-\rho_A}}^2,
    \end{split}\end{equation}
    we apply Proposition 3.1 of Sasha and get the following with probability at least $1-\delta/2$
    \begin{equation}\begin{split}
        \norm*{\b{\overline{\varSigma}_x}^{-\frac{1}{2}}\sum_{t=0}^{T-H} x_t\xi_t^\top} \leq {}& \sqrt{8\psi_{\xi}\log\b{\frac{5^n\b{1+\psi_x}^{\frac{n}{2}}}{\delta}}},
    \end{split}\end{equation}
    where $\overline{\varSigma}_x \coloneqq (T-H+1)I + \sum_{t=0}^{T-H} x_t x_t^\top$.
    Therefore,
    \begin{equation}\begin{split}
        {}& \norm*{\b{\overline{\varSigma}_x}^{-\frac{1}{2}}\b{\sum_{t=0}^{T-H} \xi_tx_t^\top + \sum_{t=0}^{T-H} x_t\xi_t^\top}\b{\overline{\varSigma}_x}^{-\frac{1}{2}}}\\
        \leq {}& 2\norm*{\b{\overline{\varSigma}_x}^{-\frac{1}{2}}\sum_{t=0}^{T-H} x_t\xi_t^\top}\norm*{\b{\overline{\varSigma}_w}^{-\frac{1}{2}}}\\
        \leq {}& \frac{4\sqrt{2}}{\sqrt{T-H+1}} \sqrt{\psi_\xi n\log\b{\frac{5\b{1+\psi_x^2}^{\frac{1}{2}}}{\delta}}},
    \end{split}\end{equation}
    which gives
    \begin{equation}\begin{split}
        \sum_{t=0}^{T-H} \xi_tx_t^\top + \sum_{t=0}^{T-H} x_t\xi_t^\top \succeq -\frac{4\sqrt{2}}{\sqrt{T-H+1}} \sqrt{\psi_\xi n\log\b{\frac{5\b{1+\psi_x^2}^{\frac{1}{2}}}{\delta}}} \overline{\varSigma}_x.
    \end{split}\end{equation}

    On the other hand, by Equation 46 of Sasha and Equation \ref{eq:cov1} in this paper, we know that with probability at least $1-\delta/2$
    \begin{equation}\begin{split}
        \frac{3}{4}\varSigma_\xi \preceq \frac{1}{T-H+1}\sum_{t=0}^{T-H}\xi_t\xi_t^\top \preceq \frac{5}{4}\varSigma_{\xi}.
    \end{split}\end{equation}

    Combining the above two events with a union bound, we get the following with probability at least $1-\delta$
    \begin{equation}\begin{split}
        \sum_{t=0}^{T-H}y_ty_t^\top = {}& \sum_{t=0}^{T-H} x_tx_t^\top + \sum_{t=0}^{T-H} \xi_t\xi_t^\top + \sum_{t=0}^{T-H} \xi_tx_t^\top + \sum_{t=0}^{T-H} x_t\xi_t^\top\\
        \succeq {}& \sum_{t=0}^{T-H}(x_tx_t^\top+\xi_t\xi_t^\top) - \frac{4\sqrt{2}}{\sqrt{T-H+1}} \sqrt{\psi_\xi n\log\b{\frac{5\b{1+\psi_x^2}^{\frac{1}{2}}}{\delta}}} \overline{\varSigma}_x\\
        \succeq {}& \b{1-\frac{4\sqrt{2}}{\sqrt{T-H+1}} \sqrt{\psi_\xi n\log\b{\frac{5\b{1+\psi_x^2}^{\frac{1}{2}}}{\delta}}}} \sum_{t=0}^{T-H}x_tx_t^\top\\
        {}& \hspace{2em} + \b{\frac{3(T-H+1)}{4}\varSigma_{\xi} - 4\sqrt{2}\sqrt{T-H+1} \sqrt{\psi_\xi n\log\b{\frac{5\b{1+\psi_x^2}^{\frac{1}{2}}}{\delta}}}} I\\
        \overset{(i)}{\succeq} {}& \frac{1}{2} \sum_{t=0}^{T-H}x_tx_t^\top + \frac{T-H+1}{2}\phi_{\xi}I\\
        \succeq {}& \frac{\phi_x+\phi_\xi}{2}(T-H+1)I.
    \end{split}\end{equation}
    Here $(i)$ is by Equation \ref{eq:cov1}.
\end{proof}

\begin{lemma}\label{lem:cov}
    The following holds for any positive integer $T$,
    \begin{equation}\begin{split}
        \norm*{\sum_{t=0}^{T}x_tx_t^\top} \leq \b{\psi_u\psi_A\psi_B\frac{3-2\rho_A}{1-\rho_A}}^2(T+1).
    \end{split}\end{equation}
\end{lemma}

\begin{proof}
    For any $t\in [1,T]$,
    \begin{equation}\begin{split}
        \norm*{x_t} \leq {}& \norm*{A^tx_0} + \sum_{\tau=0}^{t-1} \norm*{A^{t-1-\tau}Bu_{\tau}} \leq \psi_u\psi_A\rho_A^{t-1} + \psi_u\psi_B\psi_A\b{1+\sum_{\tau=0}^{t-2}\rho_A^{t-2-\tau}}\leq \psi_u\psi_B\psi_A\frac{3-2\rho_A}{1-\rho_A}.
    \end{split}\end{equation}
    Therefore, 
    \begin{equation}\begin{split}
        {}& \norm*{\sum_{t=0}^{T}x_tx_t^\top} \leq \text{tr}\b{\sum_{t=0}^{T}x_tx_t^\top} = \sum_{t=0}^{T}\text{tr}\b{x_t^{\top}x_t}\leq \b{\psi_u\psi_A\psi_B\frac{3-2\rho_A}{1-\rho_A}}^2(T+1).
    \end{split}\end{equation}
\end{proof}

\subsection{Comparison against Behavior Cloning}

We consider a special case where $H=1$, and $G_1$ is exactly related to $K$ as $G = A+BK$ (intuitively, $P$ is so large that dominates the whole term), such that the algorithm simplifies to
\begin{equation}
  \hK \gets \arg\min_{K}~ \sum_{t=0}^{T-1}  \norm*{y_{t+1} - (A+BK) y_t}_Q^2 + \norm*{v_t - K y_t}_R^2.
\end{equation}
And we compare it against the behavior cloning baseline.
\begin{equation}
  \hKBC \gets \arg\min_{K}~ \sum_{t=0}^{T-1}  \norm*{v_t - K y_t}_R^2.
\end{equation}

\begin{lemma}
  Given the same dataset $\set{y_t}$, the distances $\norm{\hK - \Ks}$ and $\norm{\hKBC - \Ks}$ shall be upper bounded by
  \begin{align*}
    \norm{\hK - \Ks} &\leq \alpha + \beta \norm{},\\
    \norm{\hKBC - \Ks} &\leq \alpha + \beta \norm{},
  \end{align*}
  such that 
\end{lemma}

\begin{proof}
  By taking the gradient we have
  \begin{equation}
    \left. \nabla_{K} \mathcal{L} \right\vert_{K = \hK} = 2 \sum_{t=0}^{T-1} B^{\top} Q((A+BK) y_t - y_{t+1}) y_t^{\top} + R (Ky_t - v_t) y_t^{\top} = 0,
  \end{equation}
  which gives
  \begin{align*}
    (B^{\top} QB + R) \hK \sum_{t=0}^{T-1} y_t y_t^{\top}
    &= B^{\top} Q \sum_{t=0}^{T-1} (y_{t+1} - Ay_t) y_t^{\top} + Rv_t y_t^{\top} \\
    &= (B^{\top} QB + R) \Ks \sum_{t=0}^{T-1} y_t y_t^{\top} + \sum_{t=0}^{T-1} \left( B^{\top} Q \xi_{t+1} - B^{\top} QA \xi_t + R \eta_t - (B^{\top} QB + R) \Ks \xi_t \right) y_t^{\top}
  \end{align*}
  Therefore, we have
  \begin{equation}
    \hK = \Ks + (B^{\top} QB + R)^{-1} \underbrace{ \sum_{t=0}^{T-1} \left( B^{\top} Q \xi_{t+1} - B^{\top} QA \xi_t + R \eta_t - (B^{\top} QB + R) \Ks \xi_t \right) y_t^{\top}}_{\theta_1} \left( \sum_{t=0}^{T-1} y_t y_t^{\top} \right)^{-1},
  \end{equation}
  and the upper bound of the distance is
  \begin{equation}
    \norm{\hK - \Ks} \leq \norm*{C} + \norm*{(B^{\top} QB + R)^{-1}} \sum_{t=0}^{T-1} \norm*{B^{\top} Q (\xi_{t+1} - A\xi_t) y_t^{\top}} \norm*{\left( \sum_{t=0}^{T-1} y_t y_t^{\top} \right)^{-1}},
  \end{equation}
  with $C = (B^{\top} QB + R)^{-1} \sum_{t=0}^{T-1} \left( R \eta_t - (B^{\top} QB + R) \Ks \xi_t \right) y_t^{\top} \left( \sum_{t=0}^{T-1} y_t y_t^{\top} \right)^{-1}$. Therefore, on expectation we have
  \begin{equation}
    \E{\norm{\hK - \Ks}} \leq \E{\norm{C}} + \norm*{(B^{\top} QB + R)^{-1}} \norm*{B^{\top} Q (I-A)} \sqrt{\tr(\varSigma_{\xi})} \sum_{t=0}^{T-1} \norm{y_t} \norm*{\left( \sum_{t=0}^{T-1} y_t y_t^{\top} \right)^{-1}}
  \end{equation}

  On the other hand, the imitation learning controller is given by
  \begin{align*}
    \hKBC
    &= \sum_{t=0}^{T-1} v_t y_t^{\top} \left( \sum_{t=0}^{T-1} y_t y_t^{\top} \right)^{-1}
    = \Ks + \sum_{t=0}^{T-1} (\eta_t - \Ks \xi_t) y_t^{\top} \left( \sum_{t=0}^{T-1} y_t y_t^{\top} \right)^{-1} \\
    &= \Ks + (B^{\top} QB + R)^{-1} 
    \left( \sum_{t=0}^{T-1} \left( (B^{\top} QB + R)\eta_t - (B^{\top} QB + R) \Ks \xi_t \right) y_t^{\top} \right) \left( \sum_{t=0}^{T-1} y_t y_t^{\top} \right)^{-1}.
  \end{align*}
  Similar to the above, we have
  \begin{equation}
    \E{\norm{\hKBC - \Ks}} \leq \E{\norm{C}} + T \norm*{(B^{\top} QB + R)^{-1}} \norm*{B^{\top} Q B} \sqrt{\tr(\varSigma_{\eta})} \norm*{\left( \sum_{t=0}^{T-1} y_t y_t^{\top} \right)^{-1}}.
  \end{equation}
  Therefore, the expected distance upper bound is smaller for the proposed algorithm than the imitation learning algorithm when
  \begin{equation}
    \norm{B^{\top} Q (I - A)} \sqrt{\tr(\varSigma_{\xi})} \leq \norm{B^{\top} QB} \sqrt{\tr(\varSigma_{\eta})}.
  \end{equation}
\end{proof}

The above result is intuitive, in the sense that when the observation noise is larger on the control input than on the state, incorporating state information is helpful.

\medskip
Then we consider a more general case where we still assume $H=1$, such that the algorithm simplifies to
\begin{equation}
  \hK , \hG_1 \gets \arg\min_{K, G_1}~ \sum_{t=0}^{T-1}  \norm*{y_{t+1} - G_1 y_t}_Q^2 + \norm*{v_t - K y_t}_R^2 + \norm*{G_1 y_t - (A+BK) y_t}_P^2.
\end{equation}

\begin{lemma}
   
\end{lemma}

\begin{proof}
  By taking the gradient we have
  \begin{equation}
    \nabla_{G_1} \mathcal{L} \biggr\vert_{\hK , \hG_1} = 2 \sum_{t=0}^{T-1} Q (\hG_1 y_t - y_{t+1}) y_t^{\top} + P (\hG_1 y_t - (A+B \hK ) y_t) y_t^{\top} = 0,
  \end{equation}
  which gives
  \begin{equation}
    (Q\hG_1 - P(\hG_1 - (A+B\hK ))) \sum_{t} y_t y_t^{\top} = Q \sum_{t} y_{t+1} y_t^{\top},
  \end{equation}
  or equivalently,
  \begin{equation}
    \hG_1 = A + B\hK + P^{-1} Q \Biggl[ \hG_1 - \underbrace{\sum_{t} y_{t+1} y_t^{\top} \left( \sum_{t} y_t y_t^{\top} \right)^{-1}}_{E} \Biggr].
  \end{equation}
  Therefore
  \begin{equation}
    \hG_1 = (1 - P^{-1} Q)^{-1} (A + B\hK - P^{-1} Q E)
    = A + B\hK - \underbrace{(1 - P^{-1} Q) P^{-1}Q (A + B\hK - E)}_{\varDelta}.
  \end{equation}
  On the other hand, we have
  \begin{align*}
    \nabla_{K} \mathcal{L} \bigg\vert_{\hK, \hG_1} &= 2 \sum_{t} R (\hK y_t - v_t) y_t^{\top} + B^{\top} P ((A+B\hK ) y_t - \hG_1 y_t) y_t^{\top} \\
    &= 2 \sum_{t} R (\hK y_t - v_t) y_t^{\top} + B^{\top} P \varDelta y_t y_t^{\top}\\
    &= 0,
  \end{align*}
  which gives
  \begin{equation}
    (R\hK + B^{\top} P\varDelta) \sum_{t} y_t y_t^{\top}
    = R \sum_{t} v_t y_t^{\top}
    = R \sum_{t} (\Ks y_t - \Ks \xi_t + \eta_t) y_t^{\top}.
  \end{equation}
  or equivalently,
  \begin{equation}
    \hK = \Ks + R^{-1} B^{\top} P \varDelta - \sum_{t} (\eta_t - \Ks \xi_t) y_t^{\top} \left( \sum_{t} y_t y_t^{\top} \right)^{-1}. 
  \end{equation}
  
\end{proof}

\subsection{Technical Lemmas}

\begin{lemma}[Matrix gradient]
  The following equations hold:
  \begin{enumerate}
      \item $\nabla_{A} \tr(A^{\top} B) = \nabla_{A} \tr(B^{\top} A) = B$.

      \item $\nabla_{A} \tr(CABA^{\top}) = CAB + C^{\top} A B^{\top}$.

      \item $\nabla_{A} \norm*{(M+NA) x - b}_R^2 = 2N^{\top} R \bigl( (M+NA)x - b \bigr) x^{\top}$.
  \end{enumerate}
  
\end{lemma}

\begin{lemma}[Norm of Random Vector]
  $\E{\norm{X}_P^2} = \tr(P \varSigma_x) + \mu^{\top} P \mu$.
\end{lemma}

\subsection{Additional Details for Numerical Experiments}\label{ssec:details}

\textit{Linear System with Linear Expert:} In this example, we first estimate the multi-step predictors by solving the following least-squares problem:  
\begin{align}
    \min_{G_{\tau, \theta}} \sum_{t=0}^{T-\tau} \|y_{t+\tau} - G_{\tau, \theta} y_t \|.
\end{align}  
Subsequently, $\hK$ is obtained using the loss~\eqref{eq:prop-loss}. Since the multi-step predictors $G_{\tau, \theta}$ remain fixed, the state prediction error is constant. Moreover, under fixed $G_{\tau, \theta}$, the optimization problem~\eqref{eq:mpc-proposed} simplifies to a quadratic program with linear constraints, admitting a closed-form solution. We apply an exponentially decaying weight of $0.9$ to the loss terms based on the prediction horizon and set $R = P = I$. 

\textit{Linear System with Nonlinear Expert:} In this case, we consider an expert modeled as a randomly initialized MLP with two hidden layers of size $16$, ReLU activations, and a final tanh activation. We assume uniform noise in both state and action measurements, bounded within $[-0.01, 0.01]$.  

For MS-PIL, the encoder is parameterized as an MLP with four hidden layers of size $512$ and leaky ReLU activations, while the multi-step predictors are implemented as a single hidden-layer MLP with a hidden size of $512$. Each model is trained for $300$ epochs with a learning rate of $5 \times 10^{-4}$. An exponentially decaying loss with parameter $\alpha = 0.9$ is applied over the time horizon. The loss weight matrices are set as $Q = 0.1\cdot I$, with $R = P = I$.

\textit{Pendulum:} We provide a custom JAX-based implementation of the inverted pendulum, similar to OpenAI Gym's \texttt{InvertedPendulum-v1} environment. For the noisy case, we introduce uniform noise in the angle, bounded within $[-1, 1]$ degrees, angular velocity noise bounded within $[-0.001, 0.001]$ degrees per second, and input noise bounded within $[-0.1, 0.1]$. The state observations are encoded using $\sin$ and $\cos$ functions for the angle. The expert is a two-layer MLP with 64 hidden units per layer, ReLU activations, and is trained using the Soft Actor-Critic (SAC) algorithm.

For MS-PIL, the encoder is parameterized as an MLP with four hidden layers of size $1024$ and leaky ReLU activations, while the multi-step predictors are implemented as a single hidden-layer MLP with a hidden size of $1024$. Each model is trained for $5000$ epochs with a cosine-scheduled learning rate starting from $5 \times 10^{-4}$ and decaying to $1 \times 10^{-8}$. An exponentially decaying loss with parameter $\alpha = 0.9$ is applied over the time horizon. The loss weight matrices are set as $Q = 0.25\cdot I$, $R = 0.01 \cdot I$, and $P = I$.

\textit{MuJoCo Experiments:} We provide a JAX-based implementation of the considered environments using XML schemes from OpenAI Gym and using MuJoCo-JAX as backend and use the \textit{implicitfast} integrator. This allows us to significantly speed up computation of dynamics functions and thus allowing to train MS-PIL and PIL in a computationally feasible manner. 

For MS-PIL, the encoder is parameterized as an MLP with four hidden layers of size $1024$ and leaky ReLU activations, while the multi-step predictors are implemented as a single hidden-layer MLP with a hidden size of $1024$. Each model is trained for $5000$ epochs with a cosine-scheduled learning rate starting from $5 \times 10^{-4}$ and decaying to $1 \times 10^{-8}$. An exponentially decaying loss with parameter $\alpha = 0.9$ is applied over the time horizon. The loss weight matrices are set as $Q = 0.25\cdot I$, $R = 0.01 \cdot I$, and $P = I$.
}{}

\end{document}